\documentclass[11pt]{article}
\title{
Private Truly-Everlasting Robust-Prediction
}

\date{January 9, 2024}
\newcommand{\remove}[1]{}

\author{Uri Stemmer\thanks{Tel Aviv University and Google research. {\tt u@uri.co.il}. Partially supported by the Israel Science Foundation (grant 1871/19) and by Len Blavatnik and the Blavatnik Family foundation.}}

\usepackage[most]{tcolorbox}
\usepackage{caption}
\usepackage{natbib}

\usepackage{algorithm}

\usepackage{scalerel}
\usepackage{mathrsfs}
\usepackage{enumitem}
\usepackage{framed}
\usepackage{bbm}
\usepackage{tikz}
\usepackage{float}
\usetikzlibrary{backgrounds}
\usepackage{color}
\usepackage{graphicx}
\usepackage{latexsym}
\usepackage{amsfonts}
\usepackage{fullpage}

\usepackage{amsmath, amssymb, amsthm}
\usepackage{multirow}
\usepackage{dsfont}
\usepackage{array}
\usepackage{hyperref}
\usepackage{adjustbox}

\def\restrict#1{\raise-.5ex\hbox{\ensuremath|}_{#1}}

\DeclareSymbolFont{AMSb}{U}{msb}{m}{n}
\DeclareMathSymbol{\N}{\mathbin}{AMSb}{"4E}
\DeclareMathSymbol{\Z}{\mathbin}{AMSb}{"5A}
\DeclareMathSymbol{\R}{\mathbin}{AMSb}{"52}
\DeclareMathSymbol{\Q}{\mathbin}{AMSb}{"51}
\DeclareMathSymbol{\erert}{\mathbin}{AMSb}{"50}
\DeclareMathSymbol{\I}{\mathbin}{AMSb}{"49}
\DeclareMathSymbol{\C}{\mathbin}{AMSb}{"43}

\newcommand{\mynote}[2]{{\textcolor{#1}{ #2}}}

\definecolor{gray}{gray}{0.4}
\newcommand{\gray}[1]{\mynote{gray}{{\footnotesize #1}}}

\newtheorem{theorem}{Theorem}[section]

\newtheorem{lemma}[theorem]{Lemma}

\newtheorem{definition}[theorem]{Definition}

\newtheorem{remark}[theorem]{Remark}

\newtheorem{proposition}[theorem]{Proposition}

\newtheorem{fact}[theorem]{Fact}

\newtheorem{corollary}[theorem]{Corollary}

\newtheorem{observation}[theorem]{Observation}

\newtheorem{example}[theorem]{Example}

\newcommand{\AAA}{\mathcal A}
\newcommand{\BBB}{\mathcal B}

\newcommand{\BbB}{\mathbb{B}}

\newcommand{\DDD}{\mathcal D}
\newcommand{\FFF}{\mathcal F}

\newcommand{\eps}{\varepsilon}

\newcommand{\error}{{\rm error}}

\newcommand{\Sim}{\texttt{Sim}}
\newcommand{\help}{\texttt{Helper}}

\newcommand{\Lap}{\operatorname{\rm Lap}}
\newcommand{\Geom}{\mathtt{Geom}}

\newcommand{\data}{\operatorname{\rm data}}
\newcommand{\VC}{\operatorname{\rm VC}}

\newcommand{\polylog}{\mathop{\rm polylog}}

\newcommand{\rectangle}{\operatorname*{\tt RECTANGLE}}
\newcommand{\rec}{\operatorname*{\tt REC}}
\newcommand{\stripe}{{\tt Stripe}}

\def\Q{\operatorname*{\mathbb{Q}}}

\def\Lap{\mathop{\rm{Lap}}\nolimits}

\makeatletter
\def\@opargbegintheorem#1#2#3{\trivlist
\item[\hskip\dimexpr\labelsep+0pt\relax{\bf #1\ #2}]({\bf #3}).\ \itshape}
\makeatother

\begin{document}

\maketitle

\begin{abstract}
Private Everlasting Prediction (PEP), recently introduced by Naor et al.\ [2023], is a model for differentially private learning in which the learner never publicly releases a hypothesis. Instead, it provides black-box access to a ``prediction oracle'' that can predict the labels of an {\em endless stream} of unlabeled examples drawn from the underlying distribution.
Importantly, PEP provides privacy both for the initial training set and for the endless stream of classification queries. We present two conceptual modifications to the definition of PEP, as well as new constructions exhibiting significant improvements over prior work. Specifically, our contributions include:
\begin{itemize}
    \item {\bf Robustness.} PEP only guarantees accuracy provided that {\em all} the classification queries are drawn from the correct underlying distribution. A few out-of-distribution queries might break the validity of the prediction oracle for future queries, even for 
    future queries which are sampled from the correct distribution. We incorporate robustness against such poisoning attacks into the definition of PEP, and show how to obtain it.

    \item {\bf Dependence of the privacy parameter $\boldsymbol{\delta}$ in the time horizon.} We present a relaxed privacy definition, suitable for PEP, that allows us to disconnect the privacy parameter $\delta$ from the number of total time steps $T$. This allows us to obtain algorithms for PEP whose sample complexity is independent from $T$, thereby making them ``truly everlasting''. This is in contrast to prior work where the sample complexity grows with $\polylog T$.

    \item {\bf New constructions.} Prior constructions for PEP exhibit sample complexity that is {\em quadratic} in the VC dimension of the target class. We present new constructions of PEP for axis-aligned rectangles and for decision-stumps that exhibit sample complexity {\em linear} in the dimension (instead of quadratic). We show that our constructions satisfy very strong robustness properties. 
\end{itemize}

\end{abstract}

\section{Introduction}

The line of work on {\em private learning}, introduced by \cite{KLNRS08}, aims to understand the computational and statistical aspects of PAC learning while providing strong privacy protections for the training data. Recall that a (non-private) PAC learner is an algorithm that takes a training set containing classified random examples and returns a hypothesis that should be able to predict the labels of fresh examples from the same distribution.
A {\em private} PAC learner must achieve the same goal while guaranteeing differential privacy w.r.t.\ the training data. This means that the produced hypothesis should not be significantly affected by any particular labeled example. Formally, the definition of differential privacy is as follows.
\begin{definition}[\cite{DMNS06}]\label{def:DP}
Let $\AAA:X^*\rightarrow Y$ be a randomized algorithm whose input is a dataset $D\in X^*$. Algorithm $\AAA$ is {\em $(\eps,\delta)$-differentially private (DP)} if for any two datasets $D,D'$ that differ on one point (such datasets are called {\em neighboring}) and for any outcome set $F\subseteq Y$ it holds that
$\Pr[\AAA(D)\in F]\leq e^{\eps}\cdot\Pr[\AAA(D')\in F]+\delta.$
\end{definition}

Unfortunately, it is now known that there are cases where private learning requires {\em significantly more resources} than non-private learning, which can be very prohibitive. 
In particular, there are simple cases where private learning is known to be {\em impossible}, even though they are trivial without the privacy constraint. Once such example is the class of all 1-dimensional threshold functions over the real line \citep{BNSV15,AlonBLMM22}. 
In addition, there are cases where private learning is possible, but provably requires significantly more runtime than non-private learning (under cryptographic assumptions) \citep{BunZ2016}. 

The line of work on {\em private prediction}, introduced by \cite{dwork2018privacy}, offers a different perspective on private learning that circumvents some these barriers. Specifically, in {\em private prediction} we consider a setting where the private learning algorithm does {\em not} output a hypothesis. Instead, it provides a black-box access to a ``prediction oracle'' that can be used to predict the labels of fresh (unlabeled) examples from the underlying distribution. Works on private prediction\footnote{See, e.g., \citep{dwork2018privacy,BassilyTT18,NandiB20,DaganF20}.} showed that, informally, for any concept class $C$ there is a private prediction algorithm that can answer $m$ ``prediction queries'' given a training set of size $\approx\VC(C)\cdot\sqrt{m}$. Note, however, that the number of ``prediction queries'' here is at most quadratic in the size of the initial training set. This %
hinders 
the acceptance of private prediction as a viable alternative to private learning, as with the latter we obtain a privatized hypothesis that could be used to predict the labels infinitely many points. There are also empirical results showing that when the number of required predictions is large, then private prediction can be inferior to classical private learning \citep{vdMaaten20}.
To tackle this, the recent work of \cite{NaorNSY23} introduced the notion of {\em private everlasting prediction (PEP)} that supports an {\em unbounded} number of prediction queries. To achieve this, \cite{NaorNSY23} allowed the content of the black-box to constantly {\em evolve} as a function of the given queries, while guaranteeing privacy both for the initial training set and for the queries. Furthermore, they proved that such an ``evolving'' black-box is necessary in order to achieve this. %
Informally, \cite{NaorNSY23} established the following theorem.
\begin{theorem}[\cite{NaorNSY23}, informal]\label{thm:NaorInformal}
For every concept class $C$ there is a private everlasting predictor using training set of size $\approx\frac{1}{\alpha\cdot\eps^2}\cdot\VC^2(C)$, where $\alpha$ is the accuracy parameter and $\eps$ is the privacy parameter (ignoring the dependence on all other parameters). The algorithm is not necessarily computationally efficient.
\end{theorem}

\subsection{Our Contributions}
Theorem~\ref{thm:NaorInformal} shows that PEP could be much more efficient than classical private learning (in terms of sample complexity). Specifically, it shows that the sample complexity of PEP is at most quadratic in the non-private sample complexity, while in the classical model for private learning the sample complexity can sometimes be arbitrarily far from that (if at all finite). 
However, a major limitation of PEP is that it only guarantees accuracy provided that {\em all} the classification queries are drawn from the correct underlying distribution. Specifically, the algorithm initially gets a labeled training set sampled according to some fixed (but unknown) distribution $\DDD$, and then answers an infinite sequence of prediction queries, {\em provided that all of these prediction queries are sampled from the same distribution $\DDD$}. This is required in order to be able to ``refhresh'' the content of the prediction oracle for supporting infinitely many queries without exhausting the privacy budget.
This could be quite limiting, where after investing efforts in training the everlasting predictor, a few out-of-distribution queries might completely contaminate the prediction oracle, rendering it useless for future queries (even if these future queries would indeed be sampled from the correct distribution). We incorporate robustness against such poisoning attacks into the definition, and 
introduce a variant of PEP which we call {\em private everlasting robust prediction (PERP)}. Informally, the requirement is that the predictor should continue to provide utility, even if only a $\gamma$ fraction of its  queries are sampled from the correct distribution, and all other queries are {\em adversarial}. We do not require the algorithm to provide accurate predictions on adversarial queries, only that these adversarial queries would not break the validity of the predictor on ``legitimate'' queries. We observe that the construction of \cite{NaorNSY23} could be adjusted to satisfy our new definition of PERP. Informally,

\begin{observation}[informal]\label{obs:genericPERP}
For every concept class $C$ there is a private everlasting {\em robust} predictor using training set of size $\approx\frac{1}{\alpha^2\cdot\gamma\cdot\eps^2}\cdot\VC^2(C)$, where $\alpha$ is the accuracy parameter, $\gamma$ is the robustness parameter, and $\eps$ is the privacy parameter (ignoring the dependence on all other parameters). The algorithm is not necessarily computationally efficient.
\end{observation}

Note that the ``price of robustness'' in this observation is roughly $\frac{1}{\alpha\gamma}$, as compared to the non-robust construction of Theorem~\ref{thm:NaorInformal}. 
We leave open the possibility of a generic construction for PERP with sample complexity linear in $\frac{1}{\alpha}$, or with sample complexity that increases slower than $\frac{1}{\gamma}$. We show that these two objectives are indeed achievable in specific cases (to be surveyed later).

\paragraph{Disconnecting the privacy parameter $\boldsymbol{\delta}$ from the time horizon.}
It is widely agreed that the definition of differential privacy only provides meaningful guarantees when the privacy parameter $\delta$ is much smaller than  $\frac{1}{n}$, where $n$ is the size of the dataset. The reason is that it is possible to satisfy the definition while leaking about $\delta n$ records in the clear (in expectation), so we want $\delta n\ll 1$ to prevent such a leakage. In the context of PEP, this means that $\delta$ should be much smaller than the total number of time steps (or queries), which we denote as $T$. The downside here is that the sample complexity of all known constructions for PEP grows with $\polylog(\frac{1}{\delta})$, which means that the sample complexity actually grows with $\polylog(T)$.\footnote{This excludes cases where (offline) private PAC learning is easy, such as learning point functions.} This means that current constructions for private everlasting prediction are not really ``everlasting'', as a finite training set allows for supporting only a bounded number of queries (sub-exponential in the size of the training set). It is therefore natural to ask if PEP could be made ``truly everlasting'', having sample complexity independent of the time horizon. 
We answer this question in the affirmative, by presenting a relaxed privacy definition (suitable for PEP) that allows us to disconnect the privacy parameter $\delta$ from the time horizon $T$.

\paragraph{New constructions for PEP.} Two shortcomings of the generic construction of \cite{NaorNSY23}, as well as our robust extension to it, are that (1) it is computationally {\em inefficient}, and (2) it exhibits sample complexity {\em quadratic} in the VC dimension of the target class. %
We present a computationally {\em efficient} construction for the class of all axis-aligned rectangles that exhibits sample complexity {\em linear} in the dimension. Furthermore, our constructing achieves very strong robustness properties, where the sample complexity grows very slowly as a function of the robustness parameter $\gamma$. %
Specifically,

\begin{theorem}[informal]\label{thm:mainIntro}
There exists a computationally efficient PERP for axis aligned rectangles in $d$ dimensions that uses sample complexity $\approx \frac{d}{\alpha\cdot\eps^2}\cdot\log^2(\frac{1}{\gamma})$, where $\alpha$ is the accuracy parameter, $\gamma$ is the robustness parameter, and $\eps$ is the privacy parameter (ignoring the dependence on all other parameters). 
\end{theorem}

Via a simple reduction, we show that our construction can be used to obtain a robust predictor also for the class of {\em $d$-dimensional decision-stumps}. (The VC dimension of this concept class is known to be $\Theta(\log d)$.) Specifically,

\begin{theorem}[informal]\label{thm:secondIntro}
There exists a computationally efficient PERP for $d$-dimensional decision-stumps that uses sample complexity $\approx \frac{\log d}{\alpha\cdot\eps^2}\cdot\log^2(\frac{1}{\gamma})$, where $\alpha$ is the accuracy parameter, $\gamma$ is the robustness parameter, and $\eps$ is the privacy parameter (ignoring the dependence on all other parameters). 
\end{theorem}

For both of these classes, classical private learning is known to be {\em impossible} (when defined over infinite domains, such as the reals). Thus, Theorems~\ref{thm:mainIntro} and~\ref{thm:secondIntro} provide further evidence that PEP could become a viable alternative to the classical model private learning: These theorems provide examples where classical private learning is impossible, while PEP is possible {\em efficiently} with sample complexity that almost matches the non-private sample complexity (while satisfying strong robustness properties).

\section{Preliminaries}

\paragraph{Notation.}
Two datasets $S$ and $S'$ are called {\em neighboring} if one is obtained from the other by adding or deleting one element, e.g., $S'=S\cup\{x'\}$.
For two random variables $Y,Z$ we write $X\approx_{(\eps,\delta)}Y$ to mean that for every event $F$ it holds that 
$\Pr[X\in F] \leq e^{\eps}\cdot\Pr[Y\in F]+\delta$, and $\Pr[Y\in F]\leq e^{\eps}\cdot\Pr[X\in F]+\delta$. Throughout the paper we assume that the privacy parameter $\eps$ satisfies $\eps=O(1)$, but our analyses extend to larger values of epsilon.

\subsection{Preliminaries from private prediction}

We now provide the formal definitions for private everlasting prediction. Before considering the utility and privacy requirements, the following definition specifies the interface, or the syntax, required from a prediction algorithm.

\begin{definition}[Prediction Oracle \cite{NaorNSY23}]\label{def:oracle}
A prediction oracle is an algorithm $\AAA$ with the following properties: 
\begin{enumerate}
	\item At the beginning of the execution, algorithm $\AAA$ receives a dataset  $S\in(X\times\{0,1\})^n$ containing $n$ labeled examples and selects a (possibly randomized) hypothesis $h_0:X\rightarrow\{0,1\}$.
	\item Then, in each round $r\in\mathbb{N}$, algorithm $\AAA$ gets a query, which is an unlabeled element $x_{r}\in X$, outputs $h_{r-1}(x_{r})$ and selects a (possibly randomized) hypothesis $h_r:X\rightarrow\{0,1\}$.
\end{enumerate}
\end{definition}

The following definition specifies the utility requirement from a prediction algorithm. Informally, the requirement is that {\em all} of the hypotheses it selects throughout the execution have low generalization error w.r.t.\ the target distribution and target concept.

\begin{definition}[Everlasting Prediction \cite{NaorNSY23}]\label{def:everlastingPrediction}
Let $\AAA$ be a prediction oracle. We say that $\AAA$ is an \emph{$(\alpha,\beta,n)$-everlasting predictor} for a concept class $C$ over a domain $X$ if the following holds for every concept $c\in C$ and for every distribution $\DDD$ over $X$. If the points in $S$ are drawn i.i.d.\ from $\DDD$ and labeled by $c$, and if all the queries $x_1,x_2,\dots$ are drawn i.i.d.\ from $\DDD$, then 
$
\Pr\left[\exists r\geq 0 \text{ s.t.\ } \error_{\DDD}(c,h_r)>\alpha \right]\leq\beta.
$ 
\end{definition}

The following definition specifies the privacy requirement from a prediction algorithm. Informally, in addition to requiring DP w.r.t.\ the initial training set, we also require that an adversary that completely determines $\AAA$'s inputs, except for the $i$th query, and gets to see all of $\AAA$'s predictions except for the predicted label for the $i$th query, cannot learn much about the $i$th query.

\begin{definition}[Private Prediction Oracle \cite{NaorNSY23}]\label{def:privatePrediction}
A prediction oracle $\AAA$ is a {\em $(\eps,\delta)$-private} if for every adversary $\BBB$ and every $T\in\mathbb{N}$, the random variables $\mbox{View}_{\BBB,T}^0$ and $\mbox{View}_{\BBB,T}^1$ (defined in Figure~\ref{fig:AdversarialExperiment}) are $(\eps,\delta)$-indistinguishable. 
\end{definition}

\begin{figure}[h!]
\begin{framed}
{\bf Parameters:} $b\in\{0,1\}$, $T\in\mathbb{N}$. \\ 

{\bf Training Phase:}

\begin{enumerate}

\item The adversary $\BBB$ chooses two labeled datasets $S^0,S^1\in(X\times\{0,1\})^*$ which are either neighboring or equal.  \quad {\small \gray{\% If $S^0,S^1$ are neighboring, then one of them can be obtained by the other by adding/removing one labeled point.}}

\item If $S^0=S^1$ then set $c_0=0$. Otherwise set $c_0=1$. \quad {\small \gray{\% We refer to the round $r$ in which $c_r=1$ as the ``challenge round''. If $c_0=1$ then the adversary chooses to place the challenge already in the initial dataset. Otherwise, the adversary will have the chance to pose a (single) challenge round in the following prediction phase.}}

\item Algorithm $\AAA$ gets $S^b$.
\end{enumerate}

{\bf Prediction phase:}
\begin{enumerate}[resume]
\item For round $r=1,2,\dots,T$:
\begin{enumerate}

    \item The adversary $\BBB$ outputs a point $x_{r}\in X$ and a bit $c_r\in\{0,1\}$, under the restriction that $\sum_{j=0}^r c_j\leq 1$.

    \item If ($c_r=0$ or $b=1$) then algorithm $\AAA$ gets $x_{r}$ and outputs a prediction $\hat{y}_{r}$. %

    \item If ($c_r=1$ and $b=0$) then algorithm $\AAA$ gets $\bot$. \quad {\small \gray{\% Here $\bot$ is a special symbol denoting ``no query''.}}
    
    \item If $c_r=0$ then the adversary $\BBB$ gets $\hat{y}_{r}$. \quad {\small \gray{\% The adversary does not get $\hat{y}_r$ if $c_r=1$.}}

\end{enumerate}

Let $\mbox{View}_{\BBB,T}^b$ be  $\BBB$'s entire view of the execution, i.e., the adversary's randomness and the sequence of predictions it received.
\end{enumerate}
\end{framed}
\vspace{-2mm}
\caption{Definition of $\mbox{View}_{\BBB,t}^0$ and $\mbox{View}_{\BBB,t}^1$.\label{fig:AdversarialExperiment}}
\end{figure}

Finally, we can define {\em private everlasting prediction} to be algorithms that satisfy both Definition~\ref{def:everlastingPrediction} (the utility requirement) and Definition~\ref{def:privatePrediction} (the privacy requirement):

\begin{definition}[Private Everlasting Prediction \cite{NaorNSY23}]
Algorithm $\AAA$ is an $(\alpha,\beta,\eps,\delta,n)$-Private Everlasting Predictor (PEP) if it is an $(\alpha,\beta,n)$-everlasting predictor and $(\eps,\delta)$-private.
\end{definition}

As we mentioned, \cite{NaorNSY23} presented a generic construction for PEP with the following properties:

\begin{theorem}[\cite{NaorNSY23}]\label{thm:Naor}
For every concept class $C$ and every $\alpha,\beta,\eps,\delta$ there is an $(\alpha,\beta,\eps,\delta,n)$-PEP for $C$ where 
$
n=\tilde{O}\left(
\frac{\VC^2(C)}{\alpha\eps^2}\cdot\polylog\left(\frac{1}{\beta\delta}\right)
\right).
$ The algorithm is not necessarily efficient.
\end{theorem}

\section{Conceptual Modifications to the Definition of PEP}

In this section we present the two conceptual modification we suggest to the definition of private everlasting prediction.

\subsection{Robustness to out-of-distribution queries}\label{sec:robustness}

We introduce a modified utility requirement for PEP in which only a $\gamma$ fraction of the queries are assumed to be sampled from the target distribution, while all other queries could be completely adversarial. This modifies only the utility requirement of PEP (Definition~\ref{def:everlastingPrediction}), and has no effect on the privacy requirement (Definition~\ref{def:privatePrediction}).

\begin{definition}[Everlasting robust prediction]\label{def:everlastingRobustPrediction}
Let $\AAA$ be a prediction oracle (as in Definition~\ref{def:oracle}). 
We say that $\AAA$ is an \emph{$(\alpha,\beta,\gamma,n)$-everlasting robust predictor} for a concept class $C$ over a domain $X$ if the following holds for every concept $c\in C$, every distribution $\DDD$ over $X$, and every adversary $\FFF$. Suppose that the points in the initial dataset $S$ are sampled i.i.d.\ from $\DDD$ and labeled correctly by $c$. Furthermore, suppose that each of the query points $x_i$ is generated as follows.  With probability $\gamma$, the point $x_i$ is sampled from $\DDD$. Otherwise (with probability $1-\gamma$), the adversary $\FFF$ chooses $x_i$ arbitrarily, based on all previous inputs and outputs of $\AAA$. Then,
$
\Pr\left[\exists r\geq 0 \text{ s.t.\ } \error_{\DDD}(c,h_r)>\alpha \right]\leq\beta.
$ 
\end{definition}

\begin{remark}
The robustness parameter $\gamma$ could actually be allowed to {\em decrease} throughout the execution. That is, as the execution progresses, the algorithm could potentially withstand growing rates of adversarial queries. For simplicity, in Definition~\ref{def:everlastingRobustPrediction} we treat $\gamma$ as remaining fixed throughout the execution.
\end{remark}

We now define private everlasting {\em robust} prediction as a combination of the privacy definition of \cite{NaorNSY23} (Definition~\ref{def:privatePrediction}) with our modified utility definition (Definition~\ref{def:everlastingRobustPrediction}).

\begin{definition}[Private Everlasting Robust Prediction]
Algorithm $\AAA$ is an $(\alpha,\beta,\gamma,\eps,\delta,n)$-Private Everlasting Robust Predictor (PERP) if it is an $(\alpha,\beta,\gamma,n)$-everlasting robust predictor (as in Definition~\ref{def:everlastingRobustPrediction}) and it is $(\eps,\delta)$-private (as in Definition~\ref{def:privatePrediction}).
\end{definition}

Armed with our new definition of PERP, we observe that the construction of \cite{NaorNSY23} extends from PEP to PERP, at the cost of inflating the sample complexity by roughly a $\frac{1}{\alpha\gamma}$ factor. Specifically,
\begin{observation}\label{obs:generic}
For every concept class $C$ and every $\alpha,\beta,\gamma,\eps,\delta$ there is an $(\alpha,\beta,\gamma,\eps,\delta,n)$-PERP for $C$ where 
$
n=\tilde{O}\left(
\frac{\VC^2(C)}{\alpha^2\cdot \gamma\cdot \eps^2}\cdot\polylog\left(\frac{1}{\beta\delta}\right)
\right).
$ The algorithm is not necessarily efficient.
\end{observation}

This observation follows from essentially the same analysis as in \citep{NaorNSY23}. Here we only flesh out the reason for the $\frac{1}{\alpha\gamma}$ blowup to the sample complexity. Informally, the generic construction of \citep{NaorNSY23} is based on the following subroutine:
\begin{enumerate}
    \item Take a training set $S$ containing $n$ labeled samples.
    \item Partition $S$ into $T\approx \frac{\alpha n}{\VC(C)}$ chunks of size $\approx\frac{\VC(C)}{\alpha}$ each.
    \item For each chunk $i\in[T]$, identify a hypothesis $h_i\in C$ that agrees with the $i$th chunk.
    \item\label{step:optimize} For $\approx \frac{\eps^2 T^2}{\alpha}$ steps: take an unlabeled query $x$ and label it using a noisy majority vote among the $T$ hypotheses.\\ \gray{\% We can support $\approx \frac{\eps^2 T^2}{\alpha}$ queries as in roughly a $(1-\alpha)$-fraction of them there will be a strong consensus among the $T$ hypotheses, to the extent that these labels would come ``for free'' in terms of the privacy analysis. So we only need to ``pay'' for about $\eps^2 T^2$ queries, which is standard using composition theorems.}
    
    \item At the end of this process, we have $\approx \frac{\eps^2 T^2}{\alpha}$ new labeled examples, which is more than $n$ provided that $n\gtrsim\frac{\VC^2(C)}{\alpha\cdot\eps^2}$. This allows us to repeat the same subroutine with these new labeled examples as our new training set.\\ { \gray{\% This is actually not accurate, as this new training set contains errors while the original training set did not. In the full construction additional steps are taken to ensure that the error does not accumulate too much from one phase to the next, making their construction inefficient.}}
\end{enumerate}

Based on this, \cite{NaorNSY23} presented a generic construction for PEP exhibiting sample complexity 
$
n=\tilde{O}\left(
\frac{\VC^2(C)}{\alpha\cdot \eps^2}\cdot\polylog\left(\frac{1}{\beta\delta}\right)
\right).
$
Now suppose that we would like to be able to withstand $(1-\gamma)$-fraction of adversarial queries. First, this prevents us from optimizing the number of supported queries in Step~\ref{step:optimize}, and we could only answer $\approx\eps^2 T^2$ queries. The reason is that with adversarial queries we can no longer guarantee that the vast majority of the queries  would be in consensus among the hypotheses we maintain. Second, only a $\gamma$-fraction of these queries would be ``true samples'', and we would need to ensure that the number of such new ``true samples'' is more than what we started with. This translate to a requirement that $n\gtrsim\frac{\VC^2(C)}{\alpha^2\cdot \gamma\cdot \eps^2}$.

\subsection{Truly everlasting private prediction}
As we mentioned in the introduction, all current constructions for PEP exhibit sample complexity that scale as $\polylog(\frac{1}{\delta})$. In addition, the definition of differential privacy is only considered adequate in this case provided that $\delta\ll \frac{1}{T}$, which means that current constructions for PEP are not ``truly everlasting'' as they require a training set whose size scales with the bound on the number of queries $T$. Formally,

\begin{definition}[Leakage attack]\label{def:leak}\;
\begin{itemize}
    \item Let $\AAA$ be an (offline) algorithm that operates on a dataset $D\in[0,1]^T$. We say that $\AAA$ is {\em leaking} if when running it on a uniformly random dataset $D$, its outcome can be post-processed to identify a point $y$ that with probability at least 1/2 satisfies $y\in D$.

\item Let $\AAA$ be an (interactive) prediction oracle, and consider an adversary $\BBB$ that interacts with $\AAA$ for $T$ time steps, and adaptively decides on $k\leq T$ steps at which the algorithm gets a random uniform sample from $[0,1]$ without the adversary learning these points directly. We refer to these points as ``challenge points''. The adversary arbitrarily chooses all other inputs to the algorithm and sees all of its outputs (the predictions it returns). We say that $\AAA$ is {\em leaking} if there is an adversary $\BBB$ that, with probability at least 1/2, at the end of the execution outputs a point $y$ that is identical to one of the challenge points.
\end{itemize}
\end{definition}

\begin{fact}
Let $\AAA$ be an algorithm that takes a dataset $D\in[0,1]^T$  and outputs every input point independently with probability $\delta$. This algorithm satisfies $(0,\delta)$-DP. Furthermore, if $T\geq\frac{1}{\delta}$, then $\AAA$ is leaking.
\end{fact}

We put forward a simple twist to the privacy definition that allows us to exclude this undesired behaviour, without inflating the sample complexity of our algorithms by a $\polylog(T)$ factor. Specifically, we require $\delta$ to {\em decrease over time}. Intuitively, instead of allowing every record $i$ to be leaked with probability $\delta$, we allow the $i$th record to be leaked with probability at most $\boldsymbol{\delta}(i)$, such that $\sum_i \boldsymbol{\delta}(i)\triangleq\delta^*$.
In the context of PEP, we show that the sample complexity only needs to grow with $\frac{1}{\delta^*}$ and not with $\max_i\{\frac{1}{\boldsymbol{\delta}(i)}\}$. This is a significant improvement: While $\max_i\{\frac{1}{\boldsymbol{\delta}(i)}\}$ must be larger than $T$ to prevent the leakage attack mentioned above, this is not the case with $\frac{1}{\delta^*}$ which can be taken to be a small constant. This still suffices in order to ensure that the algorithm is {\em not} leaking.

As a warmup, let us examine this in the context of the standard (offline) definition of differential privacy, i.e., in the context of Definition~\ref{def:DP}. 

\begin{definition}\label{def:DPnew}
Let $\eps\geq0$ be a fixed parameter and let $\boldsymbol{\delta}:\N\rightarrow[0,1]$ be a function. 
Let $\AAA:X^*\rightarrow Y$ be a randomized algorithm whose input is a dataset. Algorithm $\AAA$ is {\em $(\eps,\boldsymbol{\delta})$-DP} if for any index $i$, any two datasets $D,D'$ that differ on the $i$th entry, and any outcome set $F\subseteq Y$ it holds that
$\Pr[\AAA(D)\in F]\leq e^{\eps}\cdot\Pr[\AAA(D')\in F]+\boldsymbol{\delta}(i).$
\end{definition}

We now turn to the adaptive setting, as is required by the PEP model. Intuitively, the complication over the non-adaptive case is that the index $i$ in which the adversary poses its challenge is itself a random variable. This prevents us from directly extending Definition~\ref{def:DPnew} to the adaptive case, because Definition~\ref{def:DPnew} uses the index $i$ to quantify the closeness between the two ``neighboring distributions'', and if $i$ is a random variable then these two ``neighboring distributions'' are not well-defined until $i$ is set. A direct approach for handling with this is to require that a similar definition holds for all conditioning on $i$. This is captured in the following definition, where we  
use the the same terminology as in Definition~\ref{def:privatePrediction} (the privacy definition for PEP), and use $r^*$ to denote the challenge round, i.e., the round $r$ in which $c_r=1$.

\begin{definition}\label{def:decayingDelta}
Let $\eps\geq0$ be a fixed parameter and let $\boldsymbol{\delta}:\N\rightarrow[0,1]$ be a function.
A prediction oracle $\AAA$ is a {\em $(\eps,\boldsymbol{\delta})$-private} if for every adversary $\BBB$, every $T\in\mathbb{N}$, and every $i\in[T]$, conditioned on $r^*=i$ then the random variables $\mbox{View}_{\BBB,T}^0$ and $\mbox{View}_{\BBB,T}^1$ (defined in Figure~\ref{fig:AdversarialExperiment}) are $(\eps,\boldsymbol{\delta}(i))$-indistinguishable.  
\end{definition}

Note that this strictly generalizes Definition~\ref{def:privatePrediction}. Indeed, by taking $\boldsymbol{\delta}$ to be fixed, i.e., $\boldsymbol{\delta}(i)=\delta^*$ for all $i$, we get that $\mbox{View}_{\BBB,T}^0$ and $\mbox{View}_{\BBB,T}^1$ are $(\eps,\delta^*)$-indistinguishable (without the conditioning). To see this, note that for any event $F$ we have
\begin{align*}
\Pr[\mbox{View}_{\BBB,T}^0 \in F] &= \sum_{i}\Pr[r^*=i]\cdot\Pr[\mbox{View}_{\BBB,T}^0 \in F | r^*=i]\\
&\leq e^\eps\left(\sum_{i}\Pr[r^*=i]\cdot\Pr[\mbox{View}_{\BBB,T}^1 \in F | r^*=i]\right)+\delta^*\\
&=e^\eps \cdot \Pr[\mbox{View}_{\BBB,T}^1 \in F]+\delta^*.
\end{align*}

So Definition~\ref{def:decayingDelta} is not weaker than Definition~\ref{def:privatePrediction}. The other direction is not true. 
Specifically, consider an algorithm $\AAA$ that with probability $\delta$ declares (upon its instantiation) that it is going to publish the first record in the clear. Otherwise the algorithm publishes nothing. This is typically something that would not be considered as a ``privacy violation'', as this catastrophic event happens only with probability $\delta$. Indeed, this algorithm would satisfy Definition~\ref{def:privatePrediction}. However, with Definition~\ref{def:decayingDelta}, an adaptive attacker might decide to pose its challenge on the first input {\em if and only if} the algorithm has issued such a declaration. Otherwise the attacker never poses its challenge on the first record. In this case, in the conditional space where $r^*=1$, we have that the attacker succeeds {\em with probability one}, and so this algorithm cannot satisfy Definition~\ref{def:decayingDelta}.
This behavior makes Definition~\ref{def:decayingDelta} somewhat harder to work with. 
We propose the following relaxed variant of Definition~\ref{def:decayingDelta} that still rules out leaking algorithms.

\begin{definition}\label{def:decayingDeltaVer2}
Let $\eps\geq0$ be a fixed parameter and let $\boldsymbol{\delta}:\N\rightarrow[0,1]$ be a function.
A prediction oracle $\AAA$ is a {\em $(\eps,\boldsymbol{\delta})$-private} if for every adversary $\BBB$, every $T\in\mathbb{N}$, every $i\in[T]$, and every event $F$ it holds that
$$
\Pr[\mbox{View}_{\BBB,T}^0\in F \text{ and } r^*=i]\leq e^{\eps}\cdot
\Pr[\mbox{View}_{\BBB,T}^1\in F \text{ and } r^*=i]+\boldsymbol{\delta}(i),
$$
and vice versa.
\end{definition}
This definition is strictly weaker (provides less privacy) than Definition~\ref{def:decayingDelta}. In particular, via the same calculation as above, when taking $\boldsymbol{\delta}\equiv\delta^*$ it only implies $(\eps,\sum_i\delta^*)$-DP rather than $(\eps,\delta^*)$-DP. Nevertheless, when $\sum_i \boldsymbol{\delta}(i)\ll1$, this definition still prevents the algorithm from leaking, even if $T$ is much larger than $\left(\sum_i \boldsymbol{\delta}(i)\right)^{-1}$.

\begin{observation}
Let $\AAA$ be an $(\eps,\boldsymbol{\delta})$-private algorithm (as in Definition~\ref{def:decayingDeltaVer2}), where $\sum_i \boldsymbol{\delta}(i)<\frac18$. Then $\AAA$ is not leaking.
\end{observation}

\begin{proof}
Assume towards contradiction that $\AAA$ is leaking, and let $\BBB$ be an appropriate adversary (as in Definition~\ref{def:leak}). Then, there must exist an index $i$ such that with probability at least $4\boldsymbol{\delta}(i)$ it holds that (1) this coordinate is chosen to be a challenge point, and (2) the adversary recovers this point. Otherwise, by a union bound, the probability of $\BBB$ succeeding in 
its leakage attack would be at most $\sum_i 4\boldsymbol{\delta}(i)<\frac{1}{2}$. Let $i$ be such an index, and consider a modified adversary $\BbB$ that interacts with $\AAA$ (as in Figure~\ref{fig:AdversarialExperiment}) while simulating $\BBB$. More specifically, the adversary $\BbB$ first samples $T$ uniformly random points $x_1,\dots,x_T\in[0,1]$ and runs $\BBB$ internally. Whenever $\BBB$ specifies an input then $\BbB$ passes this input to $\AAA$. In rounds $j$ where $\BBB$ asks to give $\AAA$ a random point, then $\BbB$ passes $x_j$ to $\AAA$. Recall that (as in Figure~\ref{fig:AdversarialExperiment}) the adversary $\BbB$ needs to choose one round as its challenge round; this is chosen to be round $i$. In that round the adversary $\BbB$ does not get to see $\AAA$'s prediction but it still needs to pass the prediction to $\BBB$ in order to continue the simulation. It passes a random bit $y_i$ to $\BBB$ as if it is the prediction obtained from $\AAA$ (in all other rounds, $\BbB$ passes the predictions obtained from $\AAA$ to $\BBB$). Now, if 
the bit $b$ in the experiment specified by Figure~\ref{fig:AdversarialExperiment} is 1, and if $y_i$ is equal to the prediction returned by $\AAA$ for the $i$th bit (which happens with probability $1/2$), then this experiment perfectly simulates the interaction between $\BBB$ and $\AAA$. Hence, whenever $b=1$ then at the end of the execution $\BBB$ guesses $x_i$ with probability at least $2\boldsymbol{\delta}(i)$. 
On the other hand, if 
the bit $b$ in the experiment specified by Figure~\ref{fig:AdversarialExperiment} is 0, then $\BBB$ gets no information about $x_i$ and hence guesses it with probability exactly 0. This contradicts the definition of $(\eps,\boldsymbol{\delta})$-privacy.
\end{proof}

\section{New Constructions for PEP}

In this section we present constructions for PEP for specific concept classes that outperform the generic construction of \cite{NaorNSY23} (and our robust extension of it) on three aspects: (1) our constructions are computationally efficient; (2) our constructions exhibit sample complexity {\em linear} in the VC dimension rather than quadratic; and (3) our constructions achieve significantly stronger robustness guarantees than our extension to the construction of \cite{NaorNSY23}. We first introduce additional preliminaries that are needed for our constructions.

\subsection{Additional preliminaries}\label{sec:prelimsMore}

\paragraph{The Reorder-Slice-Compute (RSC) paradigm.}
Consider a case in which we want to apply several privacy-preserving computations to our dataset, where each computation is applied to a {\em disjoint part} (slice) of the dataset. If the slices are selected in a data-independent way, then a straightforward analysis shows that our privacy guarantees do not deteriorate with the number of slices. The following algorithm (Algorithm \texttt{RSC}) extends this to the more complicated case where we need to select the slices adaptively in a way that depends on the outputs of prior steps. 

\begin{algorithm}[H]
\caption{\bf \texttt{RSC (Reorder-Slice-Compute)} \citep{Cohen0NSS23}}\label{alg:RSC}
{\bf Input:} Dataset $D \in X^n$, integer $\tau\ge 1$, and privacy parameters $0<\eps,\delta < 1$.

\vspace{5px}

{\bf For $\boldsymbol{i=1,\dots, \tau}$ do:}

\begin{enumerate}[topsep=-1pt,rightmargin=5pt,itemsep=-1pt]%

    \item Receive a number $m_i\in \mathbb{N}$, an ordering $\prec^{(i)}$ over $X$, and an $(\eps,\delta)$-DP protocol $\AAA_i$.

    \item $\hat{m}_i \gets m_i + \Geom(1 - e^{-\eps})$
\item $S_i\gets$ the largest $\hat{m}_i$ elements in $D$ under $\prec^{(i)}$
\item $D \gets D \setminus S_i$
\item Instantiate $\AAA$ on $S_i$ and provide external access to it (via its query-answer interface).

\end{enumerate}

\end{algorithm}

\begin{theorem}[\cite{Cohen0NSS23}]
For every $\hat{\delta} > 0$, Algorithm \texttt{RSC} is $(O(\eps\log(1/\hat{\delta})),\hat{\delta}+2\tau\delta))$-DP.
\end{theorem}

\paragraph{Algorithm \texttt{Stopper}.}  The following algorithm (Algorithm \texttt{Stopper}) and its corresponding theorem (Theorem~\ref{thm:stopper}) follow as a special case of the classical \texttt{AboveThreshold}
algorithm of \cite{DNRRV09}, also known as the Sparse Vector Technique. It allows us to continually monitor a bit stream, and to indicate when the number of ones in this stream (roughly) crosses some threshold.

\begin{algorithm}[H]
\caption{\bf \texttt{Stopper}}\label{alg:stopper}
{\bf Input:} Privacy parameters $\eps,\delta$, threshold $t$, and a dataset $D$ containing input bits.

\begin{enumerate}[topsep=-1pt,rightmargin=5pt,itemsep=-1pt]%

\item {\bf In update time:} Obtain an input bit $x\in\{0,1\}$ and add $x$ to $D$.

\item {\bf In query time:}
\begin{enumerate}[topsep=-3pt,rightmargin=5pt]%
\item Let $\widetilde{\rm sum}_i\leftarrow \Lap\left(\frac{8}{\eps}\log(\frac{2}{\delta})\right)+ \sum_{x\in D} x$
\item If $\widetilde{\rm sum}_i \geq t$ then output $\top$ and HALT
\item Else output $\bot$ and CONTINUE
\end{enumerate}
\end{enumerate}
\end{algorithm}

\begin{theorem}\label{thm:stopper}
    Algorithm \texttt{Stopper}  is $(\eps,\delta)$-DP w.r.t.\ input bits (both in the initial dataset $D$ and in query times).
\end{theorem}

\paragraph{Algorithm \texttt{BetweenThresholds}.} 
The following algorithm (Algorithm \texttt{BetweenThresholds}) allows for testing a sequence of low-sensitivity queries to learn whether their values are (roughly) above or below some predefined thresholds. The properties of this algorithm are specified in Theorem~\ref{thm:BT}.

\begin{algorithm}[H]
\caption{\bf \texttt{BetweenThresholds} \citep{BunSU16}}\label{alg:BT}
{\bf Input:} Dataset $S\in X^*$, privacy parameters $\eps,\delta$, number of ``medium'' reports $k$ where $k\geq4\log(\frac{2}{\delta})$, thresholds $t_l,t_h$ satisfying $t_h-t_l\geq\frac{16}{\varepsilon}\sqrt{k \log(\frac{2}{\delta})}$, and an adaptively chosen stream of queries $f_i:X^*\rightarrow\R$ with sensitivity $1$.

\begin{enumerate}[topsep=-1pt,rightmargin=5pt,itemsep=-1pt]%
\item Let $c=0$
\item In each round $i$, when receiving a query $f_i$, do the following:
\begin{enumerate}[topsep=-3pt,rightmargin=5pt]%
\item Let $\hat{f_i}\leftarrow f_i(S)+\Lap\left(\frac{4}{\eps}\sqrt{k\log(\frac{2}{\delta})}\right)$
\item If $\hat{f_i}< t_l$ then output ``low''
\item Else if $\hat{f_i}> t_h$ then output ``high''
\item Else output ``medium'' and set $c\leftarrow c+1$. If $c=k$ then HALT.
\end{enumerate}
\end{enumerate}
\end{algorithm}

\begin{theorem}[\cite{BunSU16,TargetCharging}]\label{thm:BT}
    Algorithm \texttt{BetweenThresholds}  is $(\eps,\delta)$-DP
\end{theorem}

\paragraph{Algorithm \texttt{ChallengeBT}.} 
Note that algorithm \texttt{BetweenThresholds} halts exactly after the $k$th time that a ``medium'' answer is returned. In our application we would need a variant of this algorithm in which the halting time leaves some ambiguity as to the exact number of ``medfium'' answers obtained so far, and to whether the last provided answer was a ``medium'' answer or not. Consider algorithm \texttt{ChallengeBT}, which combines Algorithm \texttt{BetweenThresholds} with Algorithm \texttt{Stopper}.

\begin{algorithm}[H]
\caption{\bf \texttt{ChallengeBT}}\label{alg:CBT}
{\bf Input:} Dataset $S\in X^*$, privacy parameters $\eps,\delta$, number of ``medium'' reports $k$ where $k\geq4\log(\frac{4}{\delta})$, thresholds $t_l,t_h$ satisfying $t_h-t_l\geq\frac{32}{\varepsilon}\sqrt{k \log(\frac{4}{\delta})}$, a bound on the number of steps $T$, and an adaptively chosen stream of queries $f_i:X^*\rightarrow\R$ with sensitivity $1$.

\begin{enumerate}[topsep=-1pt,rightmargin=5pt,itemsep=-1pt]%
\item Instantiate \texttt{Stopper} with privacy parameters $(\eps,\delta)$, and threshold $t=k$ on the empty dataset.
\item\label{Step:ktag} Instantiate a modified version of \texttt{BetweenThresholds} on $S$ with parameters $(\eps,\frac{\delta}{2}),k'=k+\frac{8}{\eps}\log(\frac{2}{\delta})\log(\frac{T}{\delta}),t_l,t_h$. The modification to the algorithm is that the algorithm never halts (an so it does not need to maintain the counter $c$).

\item Denote ${\rm Flag}=1$.

\item Support the following queries:

\begin{enumerate}
    \item {\bf Stopping query:} Set ${\rm Flag}=1$. Query \texttt{Stopper} to get an answer $a$, and output $a$. If \texttt{Stopper} halted during this query then HALT the execution.

    \item {\bf BT query:} 
    Obtain a sensitivity-1 function $f$. If ${\rm Flag}=0$ then ignore this $f$ and do nothing. Otherwise do  the following:
    \begin{enumerate}
        \item Set ${\rm Flag}=0$.
        \item Feed $f$ to \texttt{BetweenThresholds} and receive an answer $a$. Output $a$.

        \item If $a=$``medium'' then update \texttt{Stopper} with $1$ and otherwise update it with $0$. 
    \end{enumerate}
\end{enumerate}
\end{enumerate}
\end{algorithm}

\begin{theorem}
    \texttt{ChallengeBT} is $(\eps,\delta)$-DP w.r.t.\ the input dataset $S$.
\end{theorem}

\begin{proof}
First note that if in Step\ref{Step:ktag} we were to execute \texttt{BetweenThresholds} without modifications, then the theorem would follow the privacy guarantees of \texttt{BetweenThresholds}, as \texttt{ChallengeBT} simply post-processes its outcomes. As we explain next, this modification introduces a statistical distance of at most $\frac{\delta}{2}$, and thus the theorem still holds. 
To see this, observe that algorithm \texttt{Stopper} tracks the number of ``medium'' answers throughout the execution. Furthermore, note that the value $k'$ with which we instantiate algorithm \texttt{BetweenThresholds} is noticeably bigger than $k$ (the threshold given to \texttt{Stopper}), so that with probability at least $1-\frac{\delta}{2}$ algorithm \texttt{Stopper} halts {\em before} the $k'$ time in which a ``medium'' answer is observed. When that happens, the modification we introduced to \texttt{BetweenThresholds} has no effect on the execution.\footnote{Formally, let $E$ denote the even that all of the noises sampled by \texttt{Stopper} throughout the execution are smaller than $k'-k$. We have that $\Pr[E]\geq1-\delta/2$. Furthermore, conditioned on $E$, the executions with and without the modification are identical. This shows that this modification introduces a statistical distance of at most $\delta/2$.} This completes the proof.
\end{proof}

\subsection{Axis-aligned rectangles}\label{sec:rectangles}

We are now ready to present our construction for axis-aligned rectangles in the Euclidean space $\R^d$. A concept in this class could be thought of as the product of $d$ intervals, one on each axis. Formally,

\begin{definition}[Axis-Aligned Rectangles]
Let $d\in\N$ denote the dimension. For every $w=(a_1,b_1,\dots,a_d,b_d)\in\R^{2d}$ define the concept ${\rm rec}_w:\R^d\rightarrow\{0,1\}$ as follows. For $x=(x_1,\dots,x_d)\in\R^d$ we have ${\rm rec}_w(x)=1$ iff for every $i\in[d]$ it holds that $x_i\in[a_i,b_i]$. Define the class of all axis-aligned rectangles over $\R^d$ as $\rec_d=\{{\rm rec}_w : w\in\R^{2d}\}$.
\end{definition}

\subsubsection{A simplified overview of our construction}

As we mentioned in Section~\ref{sec:robustness}, the generic construction of \cite{NaorNSY23} operates by maintaining $k\gg1$ independent hypothesis, and privately aggregating their predictions for each given query. After enough queries have been answered, then \cite{NaorNSY23} re-trains (at least) $k$ new models, treating the previously answered queries as the new training set. However, maintaining $k\gg1$ {\em complete} hypotheses can be overly expensive in terms of the sample complexity. This is the reason why their construction ended up with sample complexity {\em quadratic} in the VC dimension. We show that this can be avoided for rectangles in high dimensions. Intuitively, our gain comes from {\em not} maintaining many complete hypotheses, but rather a single ``evolving'' hypothesis.

To illustrate this, let us consider the case where $d=1$. That is, for some $a\leq b$, the target function $c^*$ is an {\em interval} of the form $c^*(x)=1$ iff $a\leq x\leq b$.  Let us also assume that the target distribution $\DDD$ is such that $0\ll\Pr_{x\sim\DDD}[c^*(x)=1]\ll1$, as otherwise the all 0 or all 1 hypothesis would be a good solution. 

Now suppose we get a sample $S$ from the underlying distribution $\DDD$, labeled by $c^*$, and let $S_{\rm left}\subseteq S$ and $S_{\rm right}\subseteq S$ be two datasets containing the $m\approx\frac{1}{\alpha\eps}$ smallest positive points $S$ and the $m$ largest positive points $S$ (respectively). We can use $S_{\rm left},S_{\rm right}$ to privately answer queries as follows: Given a query $x\in\R$, if $x$ is smaller than (almost) all of the points in $S_{\rm left}$, or larger than (almost) all of the points in $S_{\rm right}$, then we label $x$ as 0. Otherwise we label it as 1. By the Chernoff bound, this would would have error less than $\alpha$. Furthermore, in terms of privacy, we could support quite a few queries using the ``sparse vector technique'' (or the \texttt{BetweenThresholds} algorithm; see Section~\ref{sec:prelimsMore}). Informally, queries $x$ s.t.\ 
\begin{enumerate}
    \item $\max\{a\in S_{\rm left}\}\leq x\leq \min\{b\in S_{\rm right}\}$; or
    \item $x\leq \min\{a\in S_{\rm left}\}$; or
    \item $\max\{b\in S_{\rm right}\}\leq x$
\end{enumerate}
would not incur a privacy loss. Thus, in the privacy analysis, we would only need to account for queries $x$ that are ``deep'' inside $S_{\rm left}$ or ``deep'' inside $S_{\rm right}$. Recall that we aim to operate in the adversarial case, where the vast majority of the queries can be adversarial. Thus, it is quite possible that most of the queries would indeed incur such a privacy loss. However, the main observation here is that if $x$ generated a privacy loss w.r.t.\ (say) $S_{\rm left}$, then we can use $x$ to replace one of the points in $S_{\rm left}$, as it is ``deep'' inside it. So essentially every time we incur a privacy loss w.r.t.\ one of the   ``boundary datasets'' we maintain, we get a new point to add to them in exchange. We show that this balances out the privacy loss all together.  

More precisely, it is actually {\em not} true that every query that generates a privacy loss w.r.t.\ some ``boundary dataset'' could be added to it. The reason is that we want our construction to be ``everlasting'', supporting an unbounded number of queries. In this regime the \texttt{BetweenThresholds} algorithm would sometimes err and falsely identify a query $x$ as being ``deep'' inside one of our ``boundary datasets''. There would not be too many such cases, so that we can tolerate the  privacy loss incurred by such errors, but we do not want to add unrelated/adversarial queries to the ``boundary datasets'' we maintain as otherwise they will get contaminated over time. In the full construction we incorporate measures to protect against this. 
Our formal construction is given in algorithm \texttt{RectanglesPERP}. 

\begin{algorithm}
\caption{\bf \texttt{RectanglesPERP}}
{\bf Initial input:} Labeled dataset $S\in (\R^d\times\{0,1\})^n$ and parameters $\eps,\delta^*,\alpha,\beta,\gamma$.

{\bf Notations and parameters:}
For $p\in\N$ denote $\alpha_p=\alpha/2^p$, and $\beta_p=\beta/2^p$, and $\delta_p=\tilde{\Theta}\left(\frac{\delta^*\gamma\alpha\eps^2\beta}{d\cdot 8^p}\right)$, and $m_p=\Theta\left( \frac{1}{\eps^2}\log^5\left(\frac{d}{\alpha_p\cdot\beta_p\cdot\gamma\cdot\eps\cdot\delta_p}\right) \right)$, and $k_p=2\cdot m_p$, 
and $t_p=\Theta\left(\frac{d\cdot m_p}{\gamma\cdot\alpha_p}\right)$, 
and $\Delta_p=\Theta\left(\frac{\log(1/\delta_p)}{\eps}\sqrt{k_p\log(\frac{d}{\delta_p})}\log\left(\frac{d\cdot t_p}{\beta_p}\right)\right)$.

\begin{enumerate}[topsep=-1pt,rightmargin=5pt,itemsep=-1pt]%
\item Let $p=1$. \quad {\small \gray{\% $p$ denotes the current ``phase'' of the execution.}} 
\item Instantiate algorithm \texttt{RSC} on $S$. %
\item\label{step:RSC-initial} For $j=1,2,\dots,d$:
     Use \texttt{RSC} to run a copy of \texttt{ChallengeBT} on the $\approx m_p$ positive examples in $S$  with largest $j$th coordinate, and another copy on the $\approx m_p$ positive examples in $S$ with smallest $j$th coordinate. Each execution with $k_p$ as the number of ``medium'' answers, with $\left(\frac{\eps}{\log(1/\delta_p)},\frac{\delta_p}{d}\right)$ as the privacy parameters, and with thresholds $t_l=\Delta_p$ and $t_h=2\Delta_p$. Denote these executions as $\texttt{Right}_j$ and $\texttt{Left}_j$, respectively.

\item Let $D=\emptyset$ and for every $j\in[d]$ let $D_j^{\rm left}=D_j^{\rm right}=\emptyset$.

\item For time $i=1,2,3,\dots$ do the following: \quad {\small \gray{\% Infinite loop for answering queries.}} 
\begin{enumerate}[topsep=-3pt,rightmargin=5pt]%

\item\label{step:stopping} For all $j\in[d]$, pose a ``stopping query'' to $\texttt{Left}_j$ and to $\texttt{Right}_j$.

\item\label{Step:afterStopping} If any of $\texttt{Left}_j$ or $\texttt{Right}_j$ has halted during Step~\ref{step:stopping} then re-execute it on (a copy of) the corresponding dataset $D_j^{\rm left}$ or $D_j^{\rm right}$, respectively, and then empty this dataset. %

\item\label{step:obtainQuery} Obtain an unlabeled query point $x_i\in X$
\item\label{step:defaultY} Set $\hat{y}_i\leftarrow0$.  \quad {\small \gray{\% The default label is zero; this might change in the following steps.}}

\item If $x_i=\bot$ the GOTO Step~\ref{Step:newPhase}. \quad {\small \gray{\% Here $\bot$ is a special symbol denoting ``no query''.}}

\item\label{step:tests} For $j=1,2,\dots,d$ do:
\begin{enumerate}
    \item Query $\texttt{Left}_j$ for the number of points in its dataset whose $j$th coordinate are bigger than $x_i[j]$, and obtain an answer $a$. 
    If $a=$``high'' then GOTO Step~\ref{step:exitInnerLoop}.
    If $a$=``medium'' then add $x_i$ to $D_j^{\rm left}$ and GOTO Step~\ref{step:output}.

    \item Query $\texttt{Right}_j$ for the number of points in its dataset whose $j$th coordinate are smaller than $x_i[j]$, and obtain an answer $a$.  
    If $a=$``high'' then GOTO Step~\ref{step:exitInnerLoop}.
    If $a$=``medium'' then add $x_i$ to $D_j^{\rm right}$ and GOTO Step~\ref{step:output}.
\end{enumerate}

\item Set $\hat{y}_i\leftarrow1$ \quad {\small \gray{\% Changing the label to 1 if all our tests succeeded.}}

\item\label{step:exitInnerLoop} Add $(x_i,\hat{y}_i)$ to $D$

\item\label{step:output} Output the predicted label $\hat{y}_i$

\item\label{Step:newPhase} If $i=\sum_{q\in[p]}t_q$ then: \quad {\small \gray{\% The current phase ends.}}

\begin{enumerate}
    \item Let $p\leftarrow p+1$.
    \item Stop all copies of $\texttt{Left}_j$ and $\texttt{Right}_j$, and instantiate algorithm \texttt{RSC} on $D$.
    \item\label{step:reRSC} Use \texttt{RSC} to obtain new executions $\texttt{Right}_j$ and $\texttt{Left}_j$ for $j\in[d]$ (as in Step~\ref{step:RSC-initial}).

     \item\label{step:eraseD} Let $D=\emptyset$ and for every $j\in[d]$ let $D_j^{\rm left}=D_j^{\rm right}=\emptyset$.
\end{enumerate}

\end{enumerate}
\end{enumerate}
\end{algorithm}

\subsubsection{Privacy analysis of \texttt{RectanglesPERP}}

\begin{theorem}
\texttt{RectanglesPERP} is $(\eps,\boldsymbol{\delta})$-private, as in Definition~\ref{def:decayingDeltaVer2}, where $\sum_i\boldsymbol{\delta}(i)\leq\delta^*$.
\end{theorem}

\begin{proof}
Using the notations of \texttt{RectanglesPEP}, we define the function $\boldsymbol{\delta}:\N\rightarrow[0,1]$ where $\boldsymbol{\delta}(i)=\delta_{p(i)}$ and where $p(i)$ denote the phase to which $i$ belongs. When the time $i$ or the phase $p$ is clear from the context, we write $\delta_p$ instead of $\delta_{p(i)}=\boldsymbol{\delta}(i)$ to simplify the notations. %
Fix $T\in\N$, fix an adversary $\BBB$ for interacting with \texttt{RectanglesPERP} as in Figure~\ref{fig:AdversarialExperiment}, and consider the random variables $\mbox{View}_{\BBB,T}^0$ and $\mbox{View}_{\BBB,T}^1$. %
Now fix an index $i\in[T]$ and let $p=p(i)$. We need to show that for every  event $F$ it holds that
$$
\Pr[\mbox{View}_{\BBB,T}^0 \in F \text{ and } c_i=1] 
\approx_{(\eps,\delta_p)}
\Pr[\mbox{View}_{\BBB,T}^1 \in F \text{ and } c_i=1].
$$
To simplify notations, for $b\in\{0,1\}$ let us define the random variable
$$
\overline{\mbox{View}}_{\BBB,T}^b = 
\begin{cases}
\mbox{View}_{\BBB,T}^b &\text{if } c_i=1\\
\bot &\text{else}
\end{cases}
$$
This is well-defined as $c_i$ is included in the view of the adversary. Using this notation, our goal is to show that 
$$
\overline{\mbox{View}}_{\BBB,T}^0 
\approx_{(\eps,\delta_p)}
\overline{\mbox{View}}_{\BBB,T}^1. 
$$
To show this, we leverage the simulation proof-technique of \cite{Cohen0NSS23}. Specifically, 
we will show that $\overline{\mbox{View}}_{\BBB,T}^b$ can be perfectly simulated by post-processing the outcome of an $(\eps,\delta_p)$-DP computation on the bit $b$. To this end, we describe two interactive mechanisms: a simulator $\Sim$ and a helper $\help$. 
The goal of the simulator is to simulate $\BBB$'s view in the game specified in Figure~\ref{fig:AdversarialExperiment}.
This simulator does not know the value of the bit $b$. Still, it executes $\BBB$ and attempts to conduct as much of the interaction as it can without knowing $b$. Only when it is absolutely necessary for the simulation to remain faithful, our simulator will pose queries to $\help$ (who knows the value of the bit $b$), and $\help$ responds with the result of a DP computation on $b$. We will show the following two statements, which imply the theorem by closure to post-processing:
\begin{enumerate}
    \item $\help$ is $(\eps,\delta_p)$-DP w.r.t.\ the bit $b$,
    \item $\Sim$ perfectly simulates $\overline{\mbox{View}}_{\BBB,T}^b$.
\end{enumerate}

The simulator $\Sim$ begins by instantiating the adversary $\BBB$ (who is expecting to interact with \texttt{RectanglesPERP} through the game specified in Figure~\ref{fig:AdversarialExperiment}). We analyze the execution in cases, according to the time in which the adversary sets $c_r=1$. At any case, if $\BBB$ poses its challenge in any round other than $i$ then the simulator outputs $\bot$, as in the definition of $\overline{\mbox{View}}_{\BBB,T}^b$.

\paragraph{Easy case: ${\boldsymbol{i=0}}$ and ${\boldsymbol{c_0=1}}$.} That is, the adversary $\BBB$ specifies two {\em neighboring} datasets $S^0\neq S^1$ in the beginning of the execution. In this case, $\Sim$ asks $\help$ to execute Step~\ref{step:RSC-initial} of \texttt{RectanglesPERP}. That is, to execute \texttt{RSC} on $S^b$ and to provide $\Sim$ with the querying interface to the resulting instantiations of \texttt{ChallengeBT}. This preserves $(\eps,\delta_p)$-DP by the privacy guarantees of \texttt{RSC}. The continuation of the execution can be done without any further access to the bit $b$. That is, the simulator $\Sim$ continues to query the instantiations of \texttt{ChallengeBT} produces by $\help$, and whenever a new instantiation is required it can run it itself since there are no more differences between the execution with $b=0$ or with $b=1$.

\paragraph{More involved case: ${\boldsymbol{i=r\geq1}}$ and ${\boldsymbol{c_r=1}}$.} In this case, the simulator $\Sim$ can completely simulate the execution of \texttt{RectanglesPERP} until the $r$th round. In particular, at the beginning of the $r$th round, the simulator is running (by itself) the existing copies of \texttt{ChallengeBT}, referred to as $\texttt{Left}_j$ and $\texttt{Right}_j$ in the code of the algorithm. During the $r$th round, the simulator obtains from the adversary the query point $x_r$, but as it does not know the bit $b$, it does not know who among $\{x_r,\bot\}$ needs to be fed to algorithm \texttt{RectanglesPERP}. 
Either way, the simulator can execute Steps~\ref{step:stopping} till~\ref{step:defaultY}, as they do not depend on the current query. Next, the simulator runs Step~\ref{step:tests} {\em assuming that the query is $x_r$.} We proceed according to the following two sub-cases:

\begin{enumerate}
    \item[{\bf 1.}]
    {\bf A ``medium'' answer is not encountered during the simulation of Step~\ref{step:tests}.} Recall that a BT-query to algorithm \texttt{ChallengeBT} only changes its inner state if the answer is ``medium''. Hence, in this case, the execution of Step~\ref{step:tests} with $x_r$ did not affect the inner states of any of the executions of \texttt{ChallengeBT}, and the simulator can continue using them. However, the executions with $x_r$ and with $\bot$ are still not synchronized, as $x_r$ would be added to the dataset $D$ in Step~\ref{step:exitInnerLoop} whereas $\bot$ would not. This means that from this moment until a new phase begins (in Step~\ref{Step:newPhase}), there are two possible options for the dataset $D$: either with $x_r$ or without it. Thus, the next time that this dataset needs to be used by algorithm \texttt{RectanglesPERP}, which happens in Step~\ref{step:reRSC}, the simulator could not do it itself. Instead, it would ask $\help$ to conduct this step for it, similarly to the easy case. This preserves $(\eps,\delta_p)$-DP. From that moment on, the simulator does not need to access the $\help$ anymore.

    \item[{\bf 2.}]
    {\bf A ``medium'' answer is encountered, say when querying $\texttt{Left}_j$.} In this case, the executions with $x_r$ and with $\bot$ differ on the following points:
    \begin{itemize}
        \item The executions differ in the inner state of $\texttt{Left}_j$. Specifically, let $m$ denote the number of ``middle'' answers obtained from the current copy of $\texttt{Left}_j$ before time $r$. After time $r$ the internal state of $\texttt{Left}_j$ differs between the two executions in that if $b=0$ then  the copy of \texttt{Stopper} inside $\texttt{Left}_j$ has a dataset with $m$ ones, while if $b=1$ then it has $m+1$ ones. The execution of \texttt{BetweenThresholds} inside $\texttt{Left}_j$ is not affected. Thus, to simulate this, the simulator $\Sim$ asks $\help$ to instantiate a copy of algorithm \texttt{Stopper} on a dataset containing $m+b$ ones. This satisfies $(\eps,\delta_p)$-DP and allows $\Sim$ to continue the simulation of $\texttt{Left}_j$ exactly.

        \item The executions also differ in that $x_r$ would be added to the dataset $D_j^{\rm left}$ during Step~\ref{step:tests} while $\bot$ would not. In other words, from this moment until $D_j^{\rm left}$ is used (in Step~\ref{Step:afterStopping}) or erased (in Step~\ref{step:eraseD}), there are to options for this dataset: either with $x_r$ or without it. Hence, if $\Sim$ need to execute \texttt{ChallengeBT} on this dataset (in Step~\ref{Step:afterStopping}) then instead of doing it itself it ask $\help$ to do it. This also satisfies $(\eps,\delta_p)$-DP. From that moment on, the simulator does not need to access the $\help$ anymore.
    \end{itemize}
 \end{enumerate}

Overall, we showed that $\Sim$ can perfectly simulate the view of the adversary $\BBB$ while asking $\help$ to execute at most two protocols, each satisfying $(\eps,\delta_p)$-DP. 
This shows that \texttt{RectanglesPERP} is $(\eps,\boldsymbol{\delta})$-private (follows from parallel composition theorems). The fact that $\sum_i\boldsymbol{\delta}(i)=\sum_p t_p\cdot\delta_p$ converges to at most $\delta^*$ follows from our choices for $\delta_p$ and $t_p$.

\end{proof}

\subsubsection{Utility analysis of \texttt{RectanglesPERP}}

Fix a target distribution $\DDD$ over $\R^d$ and fix a target rectangle $c$. Recall that $c$ is defined as the intersection of $d$ intervals, one on every axis, and denote these intervals as $\left\{[c_j^{\rm left},c_j^{\rm right}]\right\}_{j\in[d]}$. That is, for every $x\in\R^d$ we have $c(x)=1$ if and only if $\forall j\in[d]$ we have $c_j^{\rm left}\leq  x[j]\leq c_j^{\rm right}$. We write $\rectangle(c)=\{x\in\R^d : c(x)=1\}\subseteq\R^d$ to denote the positive region of $c$. We assume for simplicity that $\Pr_{x\sim\DDD}[c(x)=1]>\alpha$ (note that if this is not the case then the all-zero hypothesis is a good solution).

\begin{algorithm}[H]
\caption{\bf Defining sub-regions of $\rectangle(c)$}\label{alg:subregions}

\begin{enumerate}[topsep=-1pt,rightmargin=5pt,itemsep=-1pt]%

\item Denote $R=\rectangle(c)$. For $p\in\N$ we denote $\alpha_p=\alpha/2^p$ and $\beta_p=\beta/2^p$, where $\alpha$ and $\beta$ are the utility parameters.

\item For $p=1,2,3,\dots$

\begin{enumerate}
    \item For axis $j=1,2,3,\dots,d$:

\begin{itemize}
    \item Define $\stripe_{p,j}^{\rm left}$ to be the rectangular strip along the inside left side of $R$ which encloses exactly weight $\alpha_p/d$ under $\DDD$. 

    \item Set $R\leftarrow R\setminus\stripe_{p,j}^{\rm left}$

    \item Define $\stripe_{p,j}^{\rm right}$ to be the rectangular strip along the inside right side of $R$ which encloses exactly weight $\alpha_p/d$ under $\DDD$. 

    \item Set $R\leftarrow R\setminus\stripe_{p,j}^{\rm right}$
\end{itemize}

\end{enumerate}

\end{enumerate}
\end{algorithm}

\begin{definition}
For $p\in\N$ and $j\in[d]$ we define $\stripe_{p,j}^{\rm left}$ and $\stripe_{p,j}^{\rm right}$ according to Algorithm~\ref{alg:subregions}.
\end{definition}

\begin{remark}
In Algorithm~\ref{alg:subregions} we assume that the stripes we define has weight {\em exactly} $\alpha_p/d$. This might not be possible (e.g., if  $\DDD$ has atoms), but this technicality can be avoided using standard techniques. For example, by replacing every sample $x\in\R$ with a pair $(x,z)$ where $z$ is sampled uniformly from $[0,1]$. We ignore this issue for simplicity.
\end{remark}

\begin{definition}
Let $R\subseteq\R^d$ and let $j\in[d]$.
For a point $x\in\R$ we write $x\in_j R$ if the projection of $R$ on the $j$th axis contains $x$. 
For a vector $x\in\R^d$ we write $x\in_j R$ if $x[j]\in_j R$. For a set of points (or vectors) $S$ we write $S\subseteq_j R$ if for every $x\in S$ it holds that $x\in_j R$.
\end{definition}

\begin{definition}
Given an (interactive) algorithm $\AAA$ whose input is a dataset, we write $\data(\AAA)$ to denote the dataset on which $\AAA$ was executed.  %
\end{definition}

\begin{definition}\label{def:noiseBounds}
For $p\in\N$, let $E(p)$ denote the following event:
\begin{enumerate}
    \item All the geometric RV's sampled during phase $p$ (via \texttt{RCS}) are at most $\frac{1}{\eps}\log(\frac{1}{\delta_p})\log(\frac{4d}{\beta_p})$.
    \item All the Laplace RV's sampled during phase $p$ (via \texttt{ChallengeBT}) are at most $\Delta_p$  in absolute value, where $\Delta_p=\frac{4\log(1/\delta_p)}{\eps}\sqrt{k_p\log(\frac{2d}{\delta_p})}\log\left(\frac{8d t_p}{\beta_p}\right)$.
\end{enumerate}
\end{definition}

\begin{lemma}\label{lem:noiseBound}
For every $p\in\N$ we have
$\Pr[E(p)]\geq1-\beta_p$.
\end{lemma}

\begin{proof}
In the beginning of phase $p$ we sample $2d$ geometric RV's (via the \texttt{RCS} paradigm) with parameter $\left(1-e^{-\eps/\log(1/\delta_p)}\right)$. By the properties of the geometric distribution, with probability at least $1-\frac{\beta_p}{2}$, all of these samples are at most $\frac{1}{\eps}\log(\frac{1}{\delta_p})\log(\frac{4d}{\beta_p})$.

Throughout phase $p$ there are at most $t_p$ time steps. In every time steps we have (at most) 2 draws from the Laplace distribution in every copy of \texttt{ChallengeBT}, and there are 2d such copies. By the properties of the Laplace distribution, with probability at least $1-\frac{\beta_p}{2}$, all of these samples are at most
$\frac{4\log(1/\delta_p)}{\eps}\sqrt{k_p\log(\frac{2d}{\delta_p})}\log\left(\frac{8d t_p}{\beta_p}\right)$ in absolute value.
\end{proof}

\begin{lemma}\label{lem:induc}
If the size of the initial labeled dataset satisfies 
$n=\Omega\left( \frac{d}{\alpha\cdot\eps^2}\polylog(\frac{d}{\alpha\cdot\beta\cdot\gamma\cdot\eps\cdot\delta^*}) \right)$, then 
for every $p\in\N$, the following holds with probability at least $\left(1-2\sum_{q\in[p]}\beta_q\right)$.
\begin{enumerate}
    \item Event $E(p')$ holds for all $p'\leq p$.
    \item For every $p'\leq p$, at any moment during phase $p'$ and for every $j\in[d]$ it holds that
$$
\data\left(  \texttt{Left}_j \right)\in_j \bigcup_{q\in[p']} \stripe_{q,j}^{\rm left},
\quad\text{and}\quad
\data\left(  \texttt{Right}_j \right)\in_j \bigcup_{q\in[p']} \stripe_{q,j}^{\rm right}.
$$
\item At the end of phase $p$, for every $j\in[d]$ and $s\in\{{\rm left}, {\rm right}\}$ the dataset $D$ contains at least $2m_{p+1}$ points from $\stripe_{p+1,j}^s$ with the label 1, and all positively labeled points in $D$ belong to $\rectangle(C)$. 
\end{enumerate}
\end{lemma}

\begin{proof}
The proof of is by induction on $p$.\footnote{The base case of this induction (for $p=1$) follows from similar arguments to those given for the induction step, and is omitted for simplicity.} To this end, we assume that Items~1-3 hold for $p-1$, and show that in this case they also hold for $p$, except with probability at most $2\beta_p$. 
First, by Lemma~\ref{lem:noiseBound}, Item~1 holds with probability at least $1-\beta_p$. We proceed with the analysis assuming that this is the case.

By assumption, at the end of phase $p-1$, the dataset $D$ contains at least $2m_{p}$ points from every $\stripe_{p,j}^s$. As we assume that the noises throughout phase $p$ are bounded (i.e., that Item~1 of the lemma holds), and asserting that $m_p\geq\frac{1}{\eps}\log(\frac{1}{\delta_p})\log(\frac{4d}{\beta_p})$, we get that Item~2 of the lemma holds in the beginning of phase $p$ (in Step~\ref{step:reRSC}, when the copies of \texttt{ChallengeBT} are executed by the RSC paradigm). Throughout the phase, if we ever re-execute any of the copies of \texttt{ChallengeBT} in Step~\ref{Step:afterStopping}, say the copy $\texttt{Left}_j$, then this happens because $\texttt{Left}_j$ has produced $\approx k_p$ ``medium'' answers. By our assumption that the noise is bounded, 
and by asserting that $k_p\geq2\Delta_p$,
we get that there were at least $k_p/2$ ``medium'' answers. (Here $\Delta_p$ bounds the Laplace noise throughout the phase, see Definition~\ref{def:noiseBounds}.) Note that every query $x$ that generated such a ``medium'' answer is added to $D_j^{\rm left}$. Again by our assumption that the noise is bounded, and by asserting that $m_p\geq 4\Delta_p$, every such query $x$ satisfies $x\in_j \data(\texttt{Left}_j)$. Hence, if we re-execute \texttt{ChallengeBT} in Step~\ref{Step:afterStopping}, then the projection of its new dataset on the $j$th axis is contained within the projection of its previous dataset. Furthermore, by asserting that $k_p\geq2m_p$ we get that the new dataset contains at least $m_p$ points. This implies (by induction) that Item~2 of the lemma continues to hold throughout all of the phase.

We now proceed with the analysis of Item~3. We have already established that, conditioned on Items~1,2,3 holding for $p-1$, then with probability at least $1-\beta_p$ we have that Items~1,2 hold also for $p$. Furthermore, this probability is over the sampling of the noises throughout phase $p$. The query points which are sampled throughout phase $p$ from the target distribution $\DDD$ are independent of these events. Recall that there are $t_p$ time steps throughout phase $p$, and that roughly $\gamma$ portion of them are sampled from $\DDD$. Recall that each $\stripe_{p+1,j}^s$ has weight $\alpha_{p+1}/d$ under $\DDD$. Thus, by the Chernoff bound (and a union bound over $j,s$), if $t_p\geq\frac{8d}{\gamma \alpha_p}\ln\left(\frac{2d}{\beta_p}\right)$, then with probability at least $1-\beta_p$, throughout phase $p$ we receive at least $\frac{\gamma\cdot\alpha_{p}\cdot t_{p}}{2d}$ points from every $\stripe_{p+1,j}^s$. By the assumption that the noises are bounded (Item~1), each of these points is added to the dataset $D$ with the label 1. Thus, provided that $t_p\geq\frac{4d}{\gamma\cdot\alpha_p}\cdot m_{p+1}$, then the dataset $D$ at the end of phase $p$ contains at least $2\cdot m_{p+1}$ points from every $\stripe_{p+1,j}^s$. Note that this dataset might contain additional points:
\begin{itemize}
    \item $D$ might contain additional points with label 0, but these do not affect the continuation of the execution. 
    \item $D$ might contain additional points with label 1. However, again by the assumption that the noises are bounded, we have that our labeling error is {\em one sided}, meaning that if a point is included in $D$ with the label 1 then its true label {\em must} also be 1.
\end{itemize}
To summarize, the dataset $D$ contains at least $2\cdot m_{p+1}$ (distinct) points from every $\stripe_{p+1,j}^s$ with the label 1, and all positively labeled points it contains belong to $\rectangle(C)$. This completes the proof.
\end{proof}

\begin{theorem}
For every $\alpha,\beta,\gamma$ 
Algorithm \texttt{RectanglesPERP} is an $(\alpha,\beta,\gamma,n)$-everlasting robust predictor where 
$n=\Theta\left( \frac{d}{\alpha\cdot\eps^2}\polylog(\frac{d}{\alpha\cdot\beta\cdot\gamma\cdot\eps\cdot\delta^*}) \right).$
\end{theorem}

\begin{proof}
At time $i$ of the execution, consider the hypothesis defined by Steps~\ref{step:obtainQuery}--\ref{step:output} of \texttt{RectanglesPERP}. (This hypothesis takes an unlabeled point $x_i$ and returns a label $\hat{y}_i$.) We now claim that whenever Items~1 and~2 defined by Lemma~\ref{lem:induc} hold (which happens with probability at least $1-\beta$), then all of these hypotheses has error at most $\alpha$ w.r.t.\ the target distribution $\DDD$ and the target concept $c$. To see this, observe that whenever Items~1 and~2 hold, then algorithm \texttt{RectanglesPERP} never errs on point outside of
$$\bigcup_{\substack{p\in\N, j\in[d]\\s\in\{{\rm left}, {\rm right}\}}} \stripe_{p,j}^s,$$
which has weight at most $\alpha$ under $\DDD$.
\end{proof}

\begin{remark}
For constant $d$, the analysis outlined above holds even if the target concept $c$ is chosen adaptively based on the initial sample $S$ (and then $S$ is relabeled according to the chosen concept $c$). The only modification is that in the base case of Lemma~\ref{lem:induc}, we cannot use the Chernoff bound to show that each $\stripe_{1,j}^s$ contains at least $2\cdot m_1$ points. Instead we would rely on standard uniform convergence arguments for VC classes, showing that w.h.p.\ over sampling the unlabeled points in $S$, the probability mass of {\em every} possible rectangle is close to its empirical weight in $S$ up to a difference of $O(\alpha)$. This argument does not extend directly to super-constant values for $d$, because in the analysis we argued about ``stripes'' with weight $\approx \alpha/d$, rather than $\alpha$, but this does not change much when $d=O(1)$.
\end{remark}

\subsection{Decision-stumps}

A decision stump could be thought of as a halfspace which is aligned with one of the principal axes. Thus, decision-stumps in $d$ dimensions could be privately predicted using our construction for rectangles from Section~\ref{sec:rectangles}. But this would be extremely wasteful as the VC dimension of decision stump is known to be $\Theta(\log d)$ instead of $\Theta(d)$. We briefly describe a simple PERP for this class which is computationally-efficient and exhibits sample complexity linear in $\log(d)$. Formally,

\begin{definition}[Decision-Stumps]
Let $d\in\N$ denote the dimension. For every $j\in[d]$, $\sigma\in\{\pm1\}$, and $t\in\R$, define the concept ${\rm desision}_{j,\sigma,t}:\R^d\rightarrow\{0,1\}$ as follows. For $x=(x_1,\dots,x_d)\in\R^d$ we have ${\rm desision}_{j,\sigma,t}(x)={\rm sign}(\sigma\cdot (x_j-t))$. Define the class of all decision-stumps over $\R^d$ as ${\rm DECISION}_d=\left\{ {\rm desision}_{j,s,t} : j\in[d], \sigma\in\{\pm1\},t\in\R  \right\}$.
\end{definition}

\begin{algorithm}
\caption{\bf \texttt{DecisionPERP}}
{\bf Initial input:} Labeled dataset $S\in (\R^d\times\{0,1\})^n$ and parameters $\eps,\delta,\alpha,\beta,\gamma$.

\begin{enumerate}[topsep=-1pt,rightmargin=5pt,itemsep=-1pt]%
\item\label{step:DDexp} Use the exponential mechanism with privacy parameter $\frac{\eps}{4}$ to select a pair $(j,\sigma)\in[d]\times\{\pm1\}$ for which there is a decision stump with low empirical error on $S$.

\item Let $p$ denote the number of positively labeled points in $S$, and let $\hat{p}\leftarrow p+\Lap(\frac{4}{\eps})$.

\item Let $D$ be a dataset containing the projection of the points in $S$ onto the $j$th axis. Sort $D$ in an ascending order if $\sigma=1$, and in descending order otherwise. Relabel the first $\hat{p}$ examples in $D$ as 1 and the other examples as 0.

\item Execute algorithm \texttt{RectanglesPERP} on $D$ (as a multiset) for predicting thresholds in 1 dimension. Use parameters $\frac{\eps}{4},\frac{\delta}{2},\frac{\alpha}{2},\frac{\beta}{2},\gamma$. During the execution, pass only the $j$ coordinate of the given queries to \texttt{RectanglesPERP}.

\end{enumerate}
\end{algorithm}

\begin{theorem}
Algorithm \texttt{DecisionPERP} is an $(\alpha,\beta,\gamma,\eps,\delta,n)$-PERP for the class of $d$-dimensional decision-stumps, where
$n=\Theta\left( \frac{1}{\alpha\eps}\log(d) + \frac{1}{\alpha\cdot\eps^2}\polylog(\frac{1}{\alpha\cdot\beta\cdot\gamma\cdot\delta}) \right)$.
\end{theorem}

\begin{proof}
The privacy analysis is straightforward; follows from composition and from the fact that once $j,\sigma,\hat{p}$ are set, then changing one element of $S$ affects at most two elements in $D$. As for the utility analysis, let $\DDD$ denote the target distribution and let $(j^*,\sigma^*,t^*)$ denote the target concept. By the properties of the exponential mechanism, with probability at least $1-\frac{\beta}{6}$, the pair $(j,\sigma)$ selected in Step~\ref{step:DDexp} is such that there is a number $t$ with which $(j,\sigma,t)$ misclassifies at most $O(\frac{1}{\eps}\log \frac{d}{\beta})$ points in $S$. 
Now let $(j,\sigma,\hat{t})$ be a concept that agrees with the relabeled dataset $D$. By the properties of the Laplace distribution, with probability at least $1-\frac{\beta}{6}$, these two concepts $(j,\sigma,t),(j,\sigma,\hat{t})$ disagree on at most $O(\frac{1}{\eps}\log\frac{d}{\beta})$ points in $S$. By the triangle inequality, we hence get that $(j,\sigma,\hat{t})$ disagrees with the target concept $(j^*,\sigma^*,t^*)$ on at most $O(\frac{1}{\eps}\log\frac{d}{\beta})$ points in $S$.
Asserting that $|S|=\Omega\left(\frac{1}{\alpha\eps}\log\frac{d}{\beta}\right)$, by standard generalization arguments, with probability at least $1-\frac{\beta}{6}$, the concept $(j,\sigma,\hat{t})$ has error at most $\frac{\alpha}{2}$ w.r.t.\ the target distribution and the target concept. Finally, by the properties of \texttt{RectanglesPERP}, with probability at least $(1-\frac{\beta}{2})$ it guarantees $\frac{\alpha}{2}$-accuracy w.r.t.\ $(j,\sigma,\hat{t})$, even if $(1-\gamma)$ fraction of the queries are adversarial. The theorem now follows by the triangle inequality.
\end{proof}

\section*{Acknowledgments}
The author would like to thank Amos Beimel for helpful
discussions about Definitions~\ref{def:decayingDelta} and~\ref{def:decayingDeltaVer2}.

\bibliographystyle{plainnat}

\end{document}